\documentclass[10pt]{extarticle}
\RequirePackage{natbib}

\usepackage{amsmath,amsfonts,bm}

\def\eqref#1{equation~\ref{#1}}

\def\1{\bm{1}}

\DeclareMathAlphabet{\mathsfit}{\encodingdefault}{\sfdefault}{m}{sl}
\SetMathAlphabet{\mathsfit}{bold}{\encodingdefault}{\sfdefault}{bx}{n}

\def\gE{{\mathcal{E}}}
\def\gF{{\mathcal{F}}}

\newcommand{\E}{\mathbb{E}}

\newcommand{\R}{\mathbb{R}}

\newcommand{\etest}{\gE_{\text{test}}}
\renewcommand{\paragraph}[1]{\smallskip\noindent\textbf{#1}\;}

\DeclareMathOperator*{\argmax}{arg\,max}
\DeclareMathOperator*{\argmin}{arg\,min}

\usepackage{microtype}
\usepackage[utf8]{inputenc}
\usepackage{graphicx}
\usepackage{subfigure}
\usepackage{booktabs} 
\usepackage{amssymb, amsmath, amsthm}
\usepackage{url}
\usepackage{mathtools}
\usepackage{comment}
\usepackage{algorithm,algorithmic}
\usepackage{float}
\usepackage{caption}
\usepackage[T1]{fontenc}
\usepackage{thm-restate}
\usepackage[margin=1in]{geometry}
\usepackage[parfill]{parskip}

\usepackage[usenames]{color}
\definecolor{DarkGreen}{rgb}{0.15,0.5,0.15}
\definecolor{DarkRed}{rgb}{0.5,0.1,0.1}
\definecolor{DarkBlue}{rgb}{0.1,0.1,0.5}
\definecolor{Magenta}{rgb}{0.9,0.1,0.9}
\usepackage[pdftex]{hyperref}
\hypersetup{
    unicode=false,          
    pdftoolbar=true,        
    pdfmenubar=true,        
    pdffitwindow=false,      
    pdfnewwindow=true,      
    colorlinks=true,       
    linkcolor=DarkGreen,          
    citecolor=Magenta,        
    filecolor=DarkGreen,      
    urlcolor=DarkBlue,          
}

\usepackage{hyperref}

\newcommand{\twonorm}[1]{\ensuremath{\| #1 \|_2}}
\newcommand{\twonormsq}[1]{\ensuremath{\| #1 \|_2^2}}

\newcommand{\calE}{\mathcal{E}}
\newcommand{\calR}{\mathcal{R}}

\newcommand{\calX}{\mathcal{X}}
\newcommand{\calY}{\mathcal{Y}}
\newcommand{\sigmin}{\sigma_{\min}}
\newcommand{\sigmax}{\sigma_{\max}}

\newcommand{\hatbeta}{\hat\beta}
\newcommand{\hatbetat}{\hat\beta_t}
\newcommand{\betastarprev}{\beta^*_{t-1}}
\newcommand{\lambdat}{\lambda_t}

\newcommand{\ood}{OOD }
\newcommand{\good}{domain }

\newtheorem{lemma}{Lemma}
\newtheorem{corollary}{Corollary}
\newtheorem{proposition}{Proposition}

\begin{document}

\title{An Online Learning Approach to \\Interpolation and Extrapolation in Domain Generalization}
\author{
Elan Rosenfeld \thanks{Machine Learning Department, CMU. Email: {\tt
elan@cmu.edu}}
\and
Pradeep Ravikumar \thanks{Machine Learning Department, CMU. Email: {\tt
pradeepr@cs.cmu.edu}} 
\and
Andrej Risteski \thanks{Machine Learning Department, CMU. Email: {\tt aristesk@andrew.cmu.edu}} 
}
\date{}
\maketitle

\begin{abstract}
\noindent A popular assumption for out-of-distribution generalization is that the training data comprises sub-datasets, each drawn from a distinct distribution; the goal is then to ``interpolate'' these distributions and ``extrapolate'' beyond them---this objective is broadly known as domain generalization. A common belief is that ERM can interpolate but not extrapolate and that the latter task is considerably more difficult, but these claims are vague and lack formal justification. In this work, we recast generalization over sub-groups as an online game between a player minimizing risk and an adversary presenting new test distributions. Under an existing notion of inter- and extrapolation based on reweighting of sub-group likelihoods, we rigorously demonstrate that extrapolation is computationally much harder than interpolation, though their statistical complexity is not significantly different. Furthermore, we show that ERM---or a noisy variant---is \emph{provably minimax-optimal} for both tasks. Our framework presents a new avenue for the formal analysis of domain generalization algorithms which may be of independent interest.
\end{abstract}

\section{Introduction}

\label{sec:intro}

Modern machine learning algorithms excel when the training and test distributions match but often fail under even moderate distribution shift \citep{beery2018recognition}; learning a predictor which generalizes to distributions which differ from the training data is therefore an important task. This objective, broadly referred to as out-of-distribution (OOD) generalization, was classically explored in a setting where there is a single ``source'' training distribution and a different ``target'' test distribution. Achieving good performance in this setting is impossible in general, so researchers have formalized several possible frameworks to study. One common choice is to make specific assumptions about covariate or label shift~\citep{widmer1996learning, bickel2009discriminative, lipton2018detecting}; another approach is Distributionally Robust Optimization (DRO), where the test distribution is assumed to lie in some uncertainty set around the training distribution~\citep{bagnell2005robust, rahimian2019distributionally}.

There has been considerable recent interest in moving beyond a single source distribution, instead assuming that the set of training data is comprised of a collection of ``environments''~\citep{blanchard2011generalizing, muandet2013domain, peters2016causal} or ``groups''~\citep{hu2018does, duchi2019distributionally, sagawa2020distributionally}, each representing a distinct distribution,\footnote{Throughout this work, we use the terms ``domain'', ``distribution'', and ``environment'' interchangeably.} where the group identity of each sample may be known. Such a setting is referred to as \emph{domain generalization}. The hope is that by cleverly training on such a collection of groups, one can derive a robust predictor which will better transfer to unseen test data. Previous literature has focused exclusively on worst-case domain generalization, where the test environment is chosen to be the worst choice among a constrained set of possible test environments. It is useful to cast such a task as solving a one-shot min-max game, where the learner selects the predictor and then an adversary selects the test environment. A key specification for this game is how future test distributions depend on the training domains (i.e., the action space for the adversary).

The most immediate choice for the set of possible test environments is simply the set of training environments. More broadly, researchers have considered how to perform well when the adversary is allowed to present test distributions which ``interpolate'' the training distributions or ``extrapolate'' beyond them, but it is unclear what is the ideal formalization of such interpolations and extrapolations. A popular choice for modeling interpolation is to allow any convex combination of the training environments---this is referred to as \emph{group/sub-population shift}, and the resulting objective is known as \emph{Group Distributionally Robust Optimization} (DRO). \citet{duchi2019distributionally, sagawa2020distributionally} give efficient algorithms for solving the Group DRO objective, but a key point is that the resulting min-max objective is \emph{exactly equivalent} to when the adversary is limited to playing only the training environments. For modeling extrapolation, \citet{krueger2020out} consider ``extrapolating'' the training likelihoods (we make this formal in Section~\ref{sec:formal-notion}), but in this game the adversary's choice will still always be a vertex of the playable region. Thus, solving the one-shot min-max game under likelihood reweighting is always equivalent to simply minimizing worst-case risk on a discrete set.

In addition to this, formal analyses of these games are sparse. A common belief is that Empirical Risk Minimization (ERM) excels at interpolation but not extrapolation; it is also generally held as folklore that extrapolation is a much harder task, which is why generalization is so difficult---but these claims are understood intuitively, rather than mathematically. Further, \citet{sagawa2020distributionally} find that when using modern neural networks in the interpolation regime, explicitly solving the Group DRO objective does not yield better solutions than simple ERM with strong regularization. Thus the relative optimality of ERM and other domain generalization algorithms remains unclear. In light of these points, we begin by considering the question: \textbf{Is there an alternative to the single-round min-max game which might allow for a more in-depth analysis of the statistical and algorithmic properties of the task of domain generalization?}

One final additional caveat with this line of research is its emphasis on worst-case optimality over all possible test environments, which is often unnecessarily conservative. This is exemplified by empirical evaluations in the \ood literature: these works train a predictor on the source data and then evaluate it on \emph{a single test set} which is chosen adversarially with respect to the predictor. Such a protocol often misses the mark for realistically comparing the expected performance of different algorithms.
For example, \citet{gulrajani2021search} point out that many recent works deliberately evaluate on a single train/test environment split with an unreasonably difficult distribution shift. When averaging performance over multiple environment splits, they find that no algorithm outperforms ERM.
This adversarial analysis can indeed be appropriate for quantifying how an algorithm will perform in the worst possible case (particularly in safety-critical applications), but this frequently does not reflect a predictor's quality in the real world: when the test environments are \emph{not} chosen adversarially, a reasonable learning algorithm should be able to do significantly better. Thus the crucial distinction is that \textbf{existing frameworks are minimax because they demand good performance of an algorithm even in the worst case, not because we actually expect the test environments to be chosen adversarially.}\footnote{This is a subtle point which we discuss in greater detail in Section~\ref{sec:fixed-baseline}.} This suggests there is room for a more nuanced measure of \ood generalization, one which adequately captures the purpose of such algorithms---to achieve consistently good performance on all possible test distributions---and allows for a formal comparison of their performances.

In this work, we aim to address the two main gaps identified above: formalizing the difference, if any (statistical \emph{and} computational), between ERM and other \ood algorithms in both interpolation and extrapolation group shift settings; and doing so in a framework that allows us to analyze a predictor's performance on potentially non-adversarial (e.g., stochastic) future test environments. To do this, we take inspiration from the literature of online convex optimization \citep{hazan2016introduction} and ask what can be achieved in a game where the learner is allowed to repeatedly refine their predictor upon observing new environments. Our analysis therefore captures an algorithm's ability to \emph{learn and adapt} from multiple training distributions to suffer less under distribution shift and consequently perform better, on average, on future test sets. Our multi-round game generalizes existing work on \good generalization, providing new insights into the quantifiable effects of observing different environments as a function of both their number and their geometric diversity. Further, this new perspective allows for a theoretical analysis of the computational and statistical complexity of interpolation versus extrapolation, formalizing and verifying the answers to several outstanding questions which until now have only been stated intuitively.

Concretely, this work makes the following contributions:
\begin{itemize}
    \item We recast domain generalization as a repeated online game between an adversary presenting test distributions and a player minimizing \emph{cumulative regret}. This framework enables meaningful analysis beyond the single-round minimax setting, and we expect it can serve as a new approach to the formal study of the efficacy of robust \ood generalization algorithms.
    \item Under an existing notion of inter- and extrapolation, we tightly characterize their respective complexities. Specifically, we prove that i) extrapolation is indeed exponentially more difficult than interpolation in a computational sense, but ii) the statistical complexity of extrapolation is not significantly higher.
    \item For both inter- and extrapolation, we show that ERM---or a noisy variant---is \emph{provably minimax-optimal} with respect to regret, as a function of the number of environments observed. For minimizing regret over any time horizon, it is impossible to improve over ERM without additional assumptions. This result supplements recent works which support the same idea theoretically \citep{rosenfeld2021risks} and empirically \citep{gulrajani2021search} for the single-round setting.
\end{itemize}

\section{The Single-Round Domain Generalization Game}

\label{sec:formal-notion}

The key assumption of domain generalization is that the training set comprises a set of distinct domains $\calE = \{e_i\}_{i=1}^E$, each of which indexes a probability distribution $p^e$, and that the test environment will relate to these domains in some pre-specified way. Let us denote the set of such possible test distributions by $\etest$. It's common to use a minimax formulation, wherein the learner's goal is to minimize the worst-case error over the possible test distributions $\etest$. For a set of predictors $\gF$ and loss $\ell$, our goal is thus to solve the objective
\begin{align*}
    \min_{f\in\gF} \max_{e\in\etest} \E_{e}[\ell(f)].
\end{align*}
In an adversarial framework, $\etest$ is the ``playable region'' of the adversary, similar to the uncertainty set in traditional DRO. A critical ingredient of the game as noted earlier is how this set of test distributions $\etest$ depends on the training domains $\calE$. It is typically presented as belonging to one of two distinct settings: \emph{interpolation} and \emph{extrapolation}. Intuitively, the interpolation setting should consist of environments which do not vary ``beyond'' the observed training environments, while the extrapolation setting should allow for such variation to some degree. However, these terms do not have a single agreed-upon meaning.

\paragraph{Formally modeling interpolation.} Given a collection of environments, there are many possible ways to consider interpolating them. In this work, we limit our analysis to the notion of likelihood reweighting which has been used previously in several works \citep{duchi2019distributionally, albuquerque2020generalizing, sagawa2020distributionally}.\footnote{As another possibility, we could directly interpolate between two samples, but this is unlikely to be meaningful for highly complex data such as images. If we were to pose a generative model, it would instead be natural to consider interpolations of the generative parameters.} We model the interpolation of a set of domains as all convex combinations (i.e., mixtures) of their likelihoods. Formally, an interpolation of the domains in $\calE$ is any distribution which is written
\begin{align}
\label{eq:dist-interpolation}
    p^{\lambda} &:= \sum_{e\in\calE} \lambda_e p^e,
\end{align}
where $\lambda\in\Delta_E$ is a vector of convex coefficients ($\Delta_E$ is the $(E-1)$-simplex). This is a fairly natural definition, as the space of interpolations is defined as the convex hull of the environments $\calE$ in distribution-space. We will denote this convex hull $\textrm{Conv}(\calE)$.

Observe that this definition is mathematically equivalent to the set of environments which can be generated via group shift, and solving the above min-max objective is precisely Group DRO. However, this notion of single-round interpolation, while perhaps intuitive, does not actually induce a more meaningful playable region for the adversary. This is because for any predictor, the optimal choice for the adversary will be whichever training environment produces the highest risk; that is, the adversary will \emph{always} play a vertex of the simplex. Thus, these two games are equivalent:
\begin{proposition}[Equivalence of interpolation and the discrete one-shot game]
\begin{align*}
\min_{f\in\gF} \max_{e\in\textrm{Conv}(\calE)} \E_{e}[\ell(f)] &\ =\  \min_{f\in\gF} \max_{e\in\calE} \E_{e}[\ell(f)].
\end{align*}
\end{proposition}
We note that in some prior work on Group DRO, learning models that minimize worst-case sub-population risk is indeed the goal---that is, they only care about test domains that match one of the source domains. In the broader domain generalization literature, however, it does not seem that this form of interpolation provides any additional constraint on \ood learning without additional regularization \citep{hu2018does}.

\paragraph{Generalizing to extrapolation.}
It is not immediately obvious how to extend this concept to include extrapolation. \citet{krueger2020out} suggest allowing for combinations in which the coefficients are still restricted to sum to 1, but may be slightly negative, where the minimum coefficient is given as a hyperparameter $\alpha$: $\sum_{e\in\calE} \lambda_e = 1, \; \lambda_e \geq -\alpha\; \forall e\in\calE$. We refer to such combinations as ``bounded affine'' combinations, and the objective they induce is equivalent to a fixed linear combination of the average loss plus the worst-case loss. It is immediate that the adversary's optimal choice is still on a vertex, so this game also reduces to minimizing over a discrete set:
\begin{proposition}[Equivalence of constraint set for extrapolation and the discrete one-shot game]
\label{prop:extr-equiv}
\begin{align*}
\min_{f\in\gF} \max_{e\in\textrm{Extr}_\alpha(\calE)} \E_{e}[\ell(f)] &= \min_{f\in\gF} \max_{e\in\calE}\  \left[(1+E\alpha) \E_{e}[\ell(f)] - \alpha \sum_{e'\in\calE} \E_{e'}[\ell(f)]\right],
\end{align*}
where $\textrm{Extr}_\alpha(\cdot)$ is all $\alpha$-bounded affine combinations.
\end{proposition} 
Thus we find that for a single round, the precise meaning of these objectives is unclear: the adversary is still choosing from a discrete set, and this model does not seem to capture the intuition that extrapolation should be fundamentally ``harder'' than interpolation. This shortcoming motivates our modified approach based on long-term regret, which we introduce shortly.

For extrapolating likelihoods, note that the resulting function is not guaranteed to be a probability distribution, as it could result in negative measure---one can instead frame it as reweighting of the environment \emph{risks} (thus in Proposition~\ref{prop:extr-equiv} above, $\E[\cdot]$ refers to general Lebesgue integration). We study this reweighting of risks in Section~\ref{sec:bounded-affine}, and we find that generalizing well over all such combinations is NP-hard. This provable difficulty in extrapolating validates our proposed sequential game, but it also indicates that additional assumptions may be necessary for modeling domain generalization. This raises interesting questions about what is the correct or most useful model of ``extrapolation'', which we do not address here.

\section{The Sequential Domain Generalization Game}

\label{sec:online-game}

We consider recasting the task of domain generalization as a continuous game of online learning in which the player is presented with sequential test domains and must refine their predictor at each round. We're therefore interested in the player's ability to \emph{learn continuously} and improve in each round. We would expect that any good learning algorithm will suffer less per distribution as we observe more of them---that is, the \emph{per-round regret} should decrease over time. Specifically, we'd like to prove a rate at which our regret goes down as a function of the number of distributions we've observed. Our game allows for an analysis of the average loss (over time) of a learning algorithm across all possible test sequences---in order to bound this performance, we consider the worst such sequence.
In Section~\ref{sec:fixed-baseline} we expound upon this idea, comparing in detail our game to existing single-round minimax settings and discussing the benefits it affords.

We now describe the game which will allow a formal analysis of the efficacy of various \good generalization strategies. The full game can be found in the box titled Algorithm~\ref{alg:game}. Note we describe a specific instance where the adversary is limited to group mixtures as described in Section~\ref{sec:formal-notion}; \textbf{the general game allows for any formally specified action space for the adversary} and we expect this will enable future analyses involving rich classes of distribution shift threat models such as $f$-divergence or $\mathcal{H}$-divergence balls \citep{bagnell2005robust, bendavid2007analysis}.

\begin{algorithm}[t]
\caption{\textbf{: Domain Generalization Game}\\ (likelihood reweighting)}
\label{alg:game}
\begin{algorithmic}
\STATE \textbf{Input:} Convex parameter space $B$, distributions $\{p^e\}_{e\in\calE}$ over $\calX \times \calY$, \\\quad strongly convex loss $\ell : B \times (\calX \times \calY) \to \R$, playable region $\Delta$.
\FOR{$t=1\ldots T$}
\STATE 1. Player chooses parameters $\hatbetat \in B$.
\STATE 2. Adversary chooses coefficients $\lambda_t \in \Delta$.
\STATE 3. Define $f_t(\beta) := \E_{(x,y)\sim p^{\lambdat}}[\ell(\beta, (x, y))] = \sum_{e\in\calE}\lambda_{t,e} \E_{(x,y)\sim p^e}[\ell(\beta, (x, y))]$.
\ENDFOR
\STATE Player suffers regret
\begin{equation*}
    R_T = \sum_{t=1}^T f_t(\hatbetat) - \min_{\beta \in B} \sum_{t=1}^T f_t(\beta).
\end{equation*}
\end{algorithmic}
\end{algorithm}

\paragraph{Game Setup.} Before the game begins, we define a family of predictors parameterized by $\beta$ lying in a convex set $B$. For some observation space $\calX$ and label space $\calY$, nature provides a fixed loss function $\ell : B \times (\calX \times \calY) \to \R$, strongly convex in the first argument, as well as a set of $E$ environments $\calE = \{e_i\}_{i=1}^E$, each of which indexes a distribution $p^e$ over $\calX \times \calY$.
We assume that $B$ is large enough such that for any $\lambda\in\Delta_E$, the parameter which minimizes risk on $p^\lambda$ lies in $B$. We further assume that for all $\beta\in B$ and $e\in\calE$, the expected loss of $\beta$ under $p^e$ is finite. The game proceeds as follows:

On round $t$, the player chooses parameters $\hatbetat \in B$. Next, the adversary chooses a set of coefficients $\lambda_t := \{\lambda_{t,e}\}_{e\in\calE}$, which defines the distribution $p^{\lambdat}$ as the weighted combination of the likelihoods of environments in $\calE$ with coefficients $\lambda_t$, as in Equation~\ref{eq:dist-interpolation}. For now, we assume that every choice of $\lambda$ by the adversary is a set of convex coefficients---that is, an  interpolation---which ensures that $p^{\lambdat}$ is a valid probability distribution; we will relax this restriction in Section~\ref{sec:bounded-affine}. At the end of the round, the player suffers loss $f_t(\hatbetat) = \calR^{\lambdat}(\hatbetat)$, defined as the risk of the predictor parameterized by $\hatbetat$ on the adversary's chosen distribution:
\begin{align*}
    \calR^{\lambdat}(\beta) &:= \E_{(x,y)\sim p^{\lambdat}}[\ell(\beta, (x, y))]
    \end{align*}
(we write $f_e = \calR^e$ for the analogous risk on distribution $p^e$). For clarity, when using the above notation we will drop the subscript $t$ when it is not necessary.

It's important to note that in this game the player does not begin ``training'' until the first round; the initial environments $\calE$ serve only to define the playable region for the adversary. Thus to recover the existing notion of single-round \good generalization, where the estimator has already seen the source environments $\calE$ and next faces an unseen test environment, the online game would actually begin with the adversary playing each of the environment distributions in $\calE$ once. As in standard online learning, our goal is to minimize \emph{regret} with respect to the best fixed predictor in hindsight after $T$ rounds. That is, we hope to minimize
\begin{align}
\label{eq:regret}
    \sum_{t=1}^T f_t(\hatbetat) - \min_{\beta\in B} \sum_{i=1}^T f_t(\beta).
\end{align}
Observe that this notion of regret straightforwardly generalizes previous work on single-round \good generalization. By allowing $T\to\infty$, we have a meaningful measure of success: each time we are presented with a new environment, we update our predictor in the hopes of improving our average performance. Crucially, this modification allows us to ask questions about the rate at which our regret decreases as a function of the number of environments observed. It also better reflects the idea that our algorithm's performance should not be evaluated in a vacuum: we aim to perform well relative to how we \emph{could} have performed over all timesteps with a single predictor.

\subsection{The Benefits of Online Regret vs. Single-Round Loss}

\label{sec:fixed-baseline}

Our focus on regret in the online setting as opposed to loss in a single round is important; it will be instructive to carefully consider the benefits to such an analysis.

\paragraph{Significance of regret with respect to a fixed baseline.} The second term in Equation~\ref{eq:regret} is crucial; the comparison to the best \emph{fixed} parameter prevents the adversary from forcing constant regret at each round and reflects the idea that we hope to eventually perform favorably compared to a single predictor which does reasonably well on all environments. Without this baseline, the player's objective would be to simply minimize the sum of the risks on all environments: $\sum_{t=1}^T f_t(\hatbetat)$. In the adversarial setting,\footnote{By this we mean the setting where the next environment is always the one which maximizes risk for the parameter chosen by the player.} the game therefore reduces to repeated, independent instances of the single-round version; clearly, the best we can do to minimize worst-case loss each single round is to play the minimax-optimal parameters $\beta^* := \argmin_{\beta\in B} \max_{\lambda\in\Delta_E} \calR^\lambda(\beta)$. In response, the adversary would always choose $\lambda^* := \argmax_{\lambda\in\Delta_E} \calR^\lambda(\beta^*)$. This game is uninteresting beyond the first round and does not adequately capture an algorithm's performance in a real-world setting where the environments are \emph{not} chosen adversarially. As mentioned in the introduction, the key observation here is that the single-round minimax framework is used to guarantee good performance even in the worst-case scenario, but we do not actually expect future test environments to be chosen in this way.

As a simple example, if we were to repeatedly play $\beta^*$ and repeatedly face the test distribution $p^*$, we should consider it more likely that this is representative of future test environments (i.e., we will continue to encounter $p^*$) than that Nature is actively trying to give us the largest possible loss. Consequently we should switch strategies and play $\argmin_{\beta\in B} \calR^{p^*}(\beta)$, which will have better performance if the pattern continues. Thus, existing frameworks overemphasize minimax performance in individual rounds---even though in reality, distribution shift is rarely adversarial---while ignoring possible improvements over time via \emph{adaptation to the changing environments}. In contrast, our longitudinal analysis allows for an algorithm to occasionally suffer preventable loss in any given turn, so long as the per-turn regret is guaranteed to decrease over time.

One particular setting where the benefits of this new framework are readily apparent is under gradual distribution shift. The single-round minimax formulation is intended for safety-critical applications where even a tiny mistake is fatal; however, when this is not the case, such an approach is far too conservative, and regret-based analyses provide a much clearer picture of expected performance. Our framework is thus not intended to supplant the single-round setting, but rather to supplement it with a new, more realistic method of formal analysis of domain generalization algorithms.

\paragraph{Implications of sublinear regret.} For any sequence of environments, there will be some parameter $\tilde\beta$ which \emph{would have} achieved the least possible cumulative loss. Sublinear regret implies that as $T \to \infty$ we will eventually recover the per-round loss of $\tilde\beta$, but without committing beforehand and with \emph{no prior knowledge} of the test environment sequence. Thus in the limit we are guaranteeing the lowest possible average loss against a fixed sequence of environments---at the same time, our analysis is minimax so as to guarantee our regret bound holds even against the worst such sequence.

Further, sublinear regret is a very powerful guarantee when the environments are stochastic, as might be expected in any real-world setting. For any prior over environment distributions $\pi(p^e)$, it is easy to see that sublinear regret implies convergence to the performance of the parameter which minimizes loss over the marginal distribution:
\begin{align*}
    \argmin_{\beta\in B} \int_{\mathcal{P}} \pi(p^e)\ \E_{p^e}[\ell(\beta, (x,y))]\ dp^e,
\end{align*}
where $\mathcal{P}$ is the set of all distributions over $\calX \times \calY$. This is because as $T\to\infty$, the $\pi$-weighted average of the sum of losses will converge to the loss on the marginal distribution---the baseline will then be whatever parameter minimizes this loss. Observe that this is strictly stronger than the guarantee of ERM, which ensures the same result only in the limit: sublinear regret implies that \emph{for every $T$}, our regret with respect to the best predictor so far is bounded as $o(T)$. Thus if by chance the distributions we've seen are not representative of the prior $\pi$ (an oft-stated motivation for \ood generalization), we are still ensuring convergence to the loss of the optimal fixed predictor in hindsight, whatever it may be. In particular, if the sequence of environments is so unfavorable that the optimal predictor in hindsight is an invariant predictor \citep{peters2016causal, arjovsky2019invariant, rosenfeld2021risks},
which ignores meaningful signal to ensure broad generalization,
sublinear regret guarantees that our algorithm's loss converges to this invariant predictor's loss.

We emphasize again that while the above example considers a stochastic adversary, \textbf{we do not in general assume a prior over environments}. Instead, we perform a minimax analysis to guard against the worst possible sequence of test distributions. We are measuring average regret with respect to \emph{time}.

\section{Theoretical Results}

\label{sec:formal-results}

Before presenting our main theoretical results, we begin with a lemma which greatly simplifies the analysis by recharacterizing the adversary's playable region.
\begin{restatable}{lemma}{fubinis}
\label{lemma:fubinis}
Recall $\calR^e(\beta)$ is defined as the risk of $\beta$ on the distribution $p^e$. Then for all $\lambda\in\Delta_E$, it holds that $\calR^{\lambda}(\beta) = \sum_{e\in\calE} \lambda_e \calR^e(\beta)$.
\end{restatable}

This reframing allows us to generalize our analysis to extrapolation without worrying that the resulting measure is not a probability distribution. Lemma~\ref{lemma:fubinis} implies that when the adversary chooses convex coefficients $\lambdat$, they are equivalently choosing a loss function $f_t$ which is a combination of $\{f_e\}_{e=1}^E$, the individual environments' risks. Each choice of $\lambdat$ uniquely defines the resulting loss function $f_t$; moving forward we will drop this explicit dependency in our notation.

\subsection{Convex Combinations}
Similar to \citet{abernethy2008optimal}, we evaluate the performance of an algorithm by defining the \emph{value} of the game after $T$ timesteps as the player's regret under optimal play by both player and adversary:
\begin{align*}
    V_T &:= \min_{\hatbeta_1\in B} \max_{\lambda_1\in\Delta_E} \ldots
    \min_{\hatbeta_T\in B} \max_{\lambda_T\in\Delta_E} \left(\sum_{t=1}^T f_t(\hatbetat) - \min_{\beta \in B} \sum_{t=1}^T f_t(\beta) \right).
\end{align*}
For fixed $T$, this allows us to formalize minimax bounds on the regret. In the traditional literature, the adversary is allowed to play losses $f_t$ from a much more general class, such as all strongly convex functions. In this setting, the value of the game in any given round $t$ is known to be exactly $V_t = \sum_{s=1}^t \frac{G_s^2}{2s\sigmin}$, where $G_s$ is the Lipshitz constant of $f_s$ at the parameter chosen by the player and $\sigmin$ is the minimum curvature of $f$.\footnote{We've omitted some details; see \citet{abernethy2008optimal} for the full result.} This means the minimax-optimal rate for regret is $\Theta(\log t)$ \citep{hazan2007logarithmic, bartlett2007adaptive}.

In contrast to traditional online learning, where the adversary is free to choose its loss from a large non-parametric class such as all strongly convex functions, our interpolation game severely restricts the adversary, allowing only convex combinations of the risks of the $E$ distributions. We might expect that such a restriction, especially when known to the player, would allow for a faster convergence to zero regret, even if the strategy which attains it is intractable. Our first result demonstrates that this is not the case.

\begin{restatable}{theorem}{convexlower}
\label{thm:lower-bound}
Suppose $\sigmax \geq \sigmin > 0$ such that $\forall e\in\calE,\ \sigmin I \preceq \nabla^2 f_e \preceq \sigmax I$. Define $g$ as the minimum gradient norm that is guaranteed to be forceable by the adversary: $g := \min_{\beta\in B}\max_{\lambda\in\Delta_E} \twonorm{\nabla f(\beta)}$. Then for all $t\in\mathbb{N}$ it holds that $
V_t > \frac{g^2\sigmin}{16\sigmax^2} \log t$.
\end{restatable}
\begin{proof}[Proof Sketch]
The general idea of the proof is to lower bound the regret on round $t$ by the optimal regret on round $t-1$ plus some additional loss suffered on round $t$. This loss depends on the distance from the chosen parameter on round $t$ to the regret minimizer for round $t-1$, as well as the adversary's choice on round $t$, and it can be bounded as $\Omega(1/t)$. By unrolling the recursion we derive an overall lower bound of order $\sum_{i=1}^t \frac{1}{i} > \log t$. The full proof can be found in Appendix~\ref{sec:thm1-proof}.
\end{proof}

Theorem~\ref{thm:lower-bound} provides insight into how the statistical complexity of generalizing to domain interpolations depends on the geometry of the source domains. Observe that the minimum forceable gradient norm $g$ encodes a sort of ``radius'' of the convex hull of loss gradients---it is easy to see that if a ball of radius $r$ can be embedded in $\textrm{Conv}(\{\nabla f_e(\beta)\}_{e=1}^E)$ then $g > r$. Thus, the restriction of the adversary to the convex hull of distributions entails a restriction on the geometry of the convex hull of the corresponding loss gradients, which subsequently determines the regret our player can be forced to suffer. The bound does not directly depend on the \emph{number} of training environments $E$; rather it scales quadratically with the size of this region, which appropriately captures the intuition that a smaller regret should be achievable for a collection of sub-distributions whose optimal parameters are very similar to one another.

With respect to the asymptotic rate of regret, this theorem provides a somewhat surprising conclusion. Even with full knowledge of the adversary's limited selection, Theorem~\ref{thm:lower-bound} shows that no algorithm can do asymptotically better than if we were playing against the more powerful adversary playing any strongly convex function. Even more interesting, this rate can be achieved with a very simple algorithm known as Follow-The-Leader (FTL), which just plays the minimizer of the sum of all previously seen functions \citep{hazan2007logarithmic}. In our game, this means playing the predictor which minimizes risk over all environments seen so far---after observing $t$ environments, FTL would therefore play
\begin{align*}
    \beta_{\textrm{FTL}} &= \argmin_\beta \sum_{s=1}^t f_s(\beta).
\end{align*}
Observe that this strategy is precisely ERM! In other words, \emph{ERM is provably minimax-optimal for interpolation}. As the adversary's playable region is a strict subset of all strongly convex functions, it is immediate that the regret suffered by playing ERM is upper bounded as $\sum_{s=1}^t G_s^2/2s\sigmin = O(\log t)$. While Theorem~\ref{thm:lower-bound} applies to the multi-round game, it has useful implications for the single-round setting. A simple corollary provides a tight bound on the attainable regret as a function of the number of environments seen. To our knowledge, this is the first such bound for single-round \good generalization.

\begin{corollary}
\label{cor:erm-opt}
Suppose we've seen $t$ environments. Then under the same setting as Theorem~\ref{thm:lower-bound}, the additional regret suffered due to one more round is $\Omega\left(\frac{1}{t}\right)$. This lower bound is attained by ERM.
\end{corollary}

\subsection{Bounded Affine Combinations}

\label{sec:bounded-affine}
One could argue that allowing the adversary only convex combinations of domains is perhaps too good to hope for. Indeed, as we've seen, ERM is optimal for such a setting, but it has been widely observed that ERM fails under minor distribution shift. We might expect that future environments would fall outside of this hull---if combinations within the hull represent a formal notion of ``interpolating'' the training distributions, then it seems our goal instead should be to ``extrapolate'' beyond them.

As discussed in Section~\ref{sec:formal-notion}, \citet{krueger2020out} consider allowing the adversary to play bounded affine combinations of the environments; while they provide no formal results for their proposed algorithm, this conceptualization of extrapolation seems a natural extension. Clearly, this game is no easier for the player---in fact, we will demonstrate that it is \emph{significantly} harder. For general Lipschitz functions, it is known that against the worst-case sequence, no deterministic strategy can guarantee sublinear regret, and attaining sublinear regret with a randomized strategy is NP-hard. Further, there is a regret lower bound of $\Omega(\sqrt{T})$ which was recently shown to be achievable with Follow-The-Perturbed-Leader (FTPL), assuming access to an optimization oracle for approximately minimizing a non-convex function \citep{suggala2020online}. As in the previous subsection, we extend these results to the task of \good generalization---that is, we demonstrate that despite the (seemingly restrictive) requirement that the adversary play bounded affine combinations of strongly convex losses that are fully known to the player, \emph{the game remains equally hard}. These results are also surprising, as an adversary that can play arbitrary Lipschitz functions is significantly more powerful than the adversary in our game.

\begin{restatable}{theorem}{deterministic}
\label{thm:no-sublinear}
No algorithm can guarantee sublinear regret against bounded affine combinations of a finite set of strongly convex losses.
\end{restatable}
\begin{proof}
We'll show that for any algorithm, there exists a sequence of loss functions chosen in response by the adversary for which the regret is bounded as $\Omega(T)$. Assume the adversary can use coefficients greater than $-\alpha$. Define
\begin{align*}
    f_{e_1}(\beta) = x^2,\qquad f_{e_2}(\beta) = \beta^4 + \frac{1}{2\alpha} \beta^2.
\end{align*}
On round $t$, our player will choose to play $\beta\in\R$. We now describe our construction of the $t$th loss in the sequence: If $|\beta| < 1$, then we choose $f_t = (1+\alpha)f_{e_1} - \alpha f_{e_2}$, and if $|\beta| \geq 1$, we choose $f_t = f_{e_1}$. In the first case, the player suffers loss $f_t(\beta) \geq 0$, and in the second case, the player suffers loss $\geq 1$. Suppose the player plays the first option $a$ times and the second option $b$ times, for a total of $a+b=T$ rounds, and suffers $\geq b$ loss.

Consider the possible best actions in hindsight. If $a \leq \frac{T}{2}$, then $\beta^* = 0$ suffers $0$ loss, meaning the player's regret is at least $b = T-a \geq \frac{T}{2}$. If, on the other hand, $a > \frac{T}{2}$, then note that for any choice $\beta$ the loss suffered is
\begin{align*}
    -a\alpha \beta^{4} + (a/2+a\alpha+b)\beta^{2} &\leq a\alpha (\beta^{2} - \beta^{4}) + (a+b) \beta^{2}
    = \left( a\alpha (1 - \beta^{2}) + 
    T \right) \beta^{2}.
\end{align*}
Choosing $\beta^* = \sqrt{1+\frac{3}{\alpha}}$ results in regret $\geq \frac{T}{2}$. In either case, the player suffers $\Omega(T)$ regret.

For completeness's sake, in Appendix~\ref{sec:thm234-proof} we also include a proof of the existence of a regression task and a set of environments which could give rise to such a set of loss functions.
\end{proof}
Thus we find that just as in the general non-convex case, a weaker adversary is necessary. In the following we consider a relaxed version with an ``oblivious'' adversary: this adversary is forced to select the entire sequence of loss functions at the beginning of the game (our lower bounds hold despite this relaxation). We might hope that against such a restricted adversary, the computational requirements of achieving sublinear regret would be lessened---perhaps there would be no need for an optimization oracle.
However, Theorem~\ref{thm:nphard} proves otherwise:
\begin{restatable}{theorem}{nphard}
\label{thm:nphard}
Against an oblivious adversary playing bounded affine combinations, achieving sublinear regret is NP-hard.
\end{restatable}
\begin{proof}
Consider the problem of identifying the maximum size of a stable set of a graph on $|V|$ vertices; such a problem is not approximable in polynomial time to within a factor $|V|^{(1/2-\epsilon)}$ for any $\epsilon > 0$ unless $NP = P$ \citep{haastad1999clique, deklerk2008complexity}. We will demonstrate that solving this problem up to a constant factor reduces to achieving sublinear regret on an online strongly convex game with bounded affine coefficients. Let $-\alpha$ represent the minimum negative coefficient allowed for the adversary. Given the graph $G$ on $|V| > 1$ vertices, denote by $A$ its adjacency matrix. Then the maximum stable set size $\gamma(G)$ can be written $\frac{1}{\gamma(G)} = \min_{\beta\in\Delta_{|V|}} \beta^T(I+A)\beta$ by a result of \citet{motzkin1965maxima}. We define a game where the adversary has two functions:
\begin{align*}
    f_{e_1}(\beta) = \frac{1}{1+\alpha} \beta^T(|V|I + A)\beta,\qquad f_{e_2}(\beta) = \frac{|V|-1}{\alpha} \twonormsq{\beta}.
\end{align*}
Note that $f_{e_1}$ is strongly convex because $(|V|-1)I+A$ is diagonally dominant and therefore PSD. Each round, the player plays some $\beta \in \Delta_{|V|}$, and the (oblivious) adversary chooses the loss
\begin{align*}
    (1+\alpha) f_{e_1} - \alpha f_{e_2} &= \beta^T(|V|I+A)\beta - (|V|-1)\twonormsq{\beta} = \beta^T(I+A)\beta.
\end{align*}
Define $L_T$ as the loss suffered by the player after $T$ rounds. Clearly, the optimal choice would be to play $\beta$ such that $\beta^T(I+A)\beta = \frac{1}{\gamma(G)}$ each round, implying that $\frac{T}{\gamma(G)} \leq L_T$ and also that regret can be written $L_T - \frac{T}{\gamma(G)}$. Suppose there exists a polynomial-time strategy with regret growing sublinearly with $T$. Then by definition, there exists a constant $T_0\in\textrm{poly}(|V|)$ such that on all rounds $T > T_0$, the player's regret is upper bounded as
\begin{align*}
    L_T - \frac{T}{\gamma(G)} \leq \frac{1}{|V|} T \leq \frac{T}{\gamma(G)} \implies L_T \leq \frac{2T}{\gamma(G)}.
\end{align*}
Putting these inequalities together, we get $\frac{1}{\gamma(G)} \leq \frac{L_T}{T} \leq \frac{2}{\gamma(G)}$, which implies $\frac{1}{2}\gamma(G) \leq \frac{T}{L_T} \leq \gamma(G)$. Recall that this holds for all $T > T_0$, so our polynomial-time algorithm has attained a $2$-approximation to the maximum stable set size.
\end{proof}
Computationally, our game of extrapolation is just as difficult as achieving sublinear regret on arbitrary Lipschitz functions. These results present, for the first time, proof of \emph{an exponential computational complexity gap between interpolation and extrapolation in the \good generalization setting}, formally verifying existing intuition.

We now turn our attention to the statistical complexity of regret minimization under bounded affine combinations. Recall that for the case of convex combinations (i.e. interpolations), Theorem~\ref{thm:lower-bound} shows a minimax lower bound of $\Omega(\log t)$ which can be achieved with standard ERM. Before we consider the bounded affine setting (i.e. extrapolations), we again note that for an adversary playing arbitrary Lipschitz functions,
\citet{suggala2020online} demonstrate that with access to a non-convex optimization oracle,
FTPL can achieve the minimax lower bound of $\Omega(\sqrt{T})$. The FTPL strategy is to play the parameter which minimizes the sum of the observed environments plus a noise term---specifically, 
FTPL takes the sum of existing risks, samples a random linear function of the parameters, and solves for the parameters which minimize this ``perturbed'' sum. In our game, then, FTPL is just a noisy variant of ERM. Computational limitations notwithstanding, the natural next question is if playing against an oblivious adversary is enough of a relaxation that we can surpass this lower bound. That is, can we outperform ERM in this setting \emph{at all?} Our final result answers this question in the negative:

\begin{restatable}{theorem}{expertadvice}
Against an oblivious adversary playing bounded affine combinations, the achievable regret is lower bounded as $\Omega(\sqrt{T})$.
\end{restatable}
\begin{proof}
For a fixed, convex loss $\ell$ and convex parameter space $\Theta$, predicting with expert advice is known to have an information-theoretic minimax regret lower bound of $\Omega(\sqrt{T})$ \citep[Theorem 3.7]{cesa2006prediction}. We will give a reduction which demonstrates that the same lower bound holds for bounded affine combinations of strongly convex losses.

Assume a fixed convex loss $\ell : \Theta \times \Theta \mapsto \R$ over convex $\Theta$ and fix the adversary's coefficient lower bound as $-\alpha$. Suppose on round $t$, we are presented with $E$ experts' predictions, which we imagine as an $E$-dimensional vector $\tilde\theta_t$ whose $i$th entry is the prediction of the $i$th expert. Define the following functions over elements $\delta\in\Delta_E$:
\begin{align*}
    f_{e_1}(\delta, \theta^*) &= \frac{1}{1+\alpha} \left[\ell(\delta^T\tilde\theta_t, \theta^*) + \twonormsq{\delta}\right], \qquad
    f_{e_2}(\delta, \theta^*) = \frac{1}{\alpha} \twonormsq{\delta}.
\end{align*}
Note that both these functions are both strongly convex in $\delta$. Consider what happens if the adversary plays $(1+\alpha) f_{e_1} - \alpha f_{e_2} = \ell$. Suppose for the sake of contradiction there exists an algorithm playing $\hat\delta_t$ which achieves $o(\sqrt{T})$ regret with respect to $\delta^*$, defined as the best fixed $\delta\in\Delta_E$ in hindsight:
\begin{align*}
    \delta^* &:= \argmin_{\delta\in\Delta_E} \sum_{t=1}^T \ell(\delta^T\tilde\theta_t, \theta^*).
\end{align*}
As this represents a convex combination of the experts' predictions, it is clear that the loss suffered by $\delta^*$ will be less than or equal to the loss suffered by the best expert. This implies that by taking this algorithm's choice $\hat\delta_t$ each round and playing $\hat\delta_t^T\tilde\theta_t$, we will achieve $o(\sqrt{T})$ regret with respect to the best expert, defying the known lower bound. It follows that the lower bound of $\Omega(\sqrt{T})$ holds even for bounded affine combinations of strongly convex functions.
\end{proof}
This theorem implies two crucial points: firstly, that \emph{ERM remains minimax optimal for this model of extrapolation}; and secondly, that \emph{proper regularization is essential for good \ood generalization}. This provides theoretical justification for the empirical findings of \citet{sagawa2020distributionally} and complements existing results on the value of explicit regularization for group shift \citep{hu2018does}. Additionally, we find that even though there is an exponential computational complexity gap between the two tasks, the statistical gap is not too large---$\Theta(\log T)$ versus $\Theta(\sqrt{T})$ regret.

\section{Related Work}

Many works provide formal guarantees for \ood generalization by assuming invariances in the causal structure of the data: a set of interventions is assumed to result in separate fixed environments \citep{peters2016causal, heinze2018invariant, heinze2020conditional, christiansen2020causal} or distribution shift over time \citep{tian2001causal, didelez2006direct}, and the test distribution will likewise represent such an intervention.
Under sufficiently strong conditions it is then possible to identify which features have invariant relationships with the target variable; recovery of these features ensures reasonable performance despite arbitrary future interventions on the other variables. However, these works assume full or partial observation of the covariates, and therefore they do not apply to the setting where the data is a complex function of unobserved latent variables.

Works which eschew a direct causal formalization often still depend upon the intuition of ``invariance'' within the context of causality. The IRM objective \citep{arjovsky2019invariant} was designed for such a setting assuming the target variables' causal mechanisms remain invariant, but it lacked serious theoretical justification; \citet{krueger2020out} likewise suggest an algorithm for extrapolation but similarly fail to provide any formal guarantees. \citet{rosenfeld2021risks} subsequently showed that, while these and other similar objectives may work under strong conditions in the linear setting, the same cannot be said for more complex data. \citet{albuquerque2020generalizing} theoretically analyze extrapolation beyond the convex hull of domain likelihoods and give generalization bound via $\mathcal{H}$-divergences. Unfortunately, this bound scales linearly with both the maximum discrepancy between pairs of training distributions and between the test distribution and training environment hull.

This work relates the nascent study of domain generalization theory to prior work on online and lifelong learning \citep{thrun1998lifelong, mitchell2015neverending, hazan2016introduction}, for which there already exist provable regret bounds and efficiency guarantees \citep{balcan2015efficient, alquier2017regret}. The main difference is that those works---which are for more general online learning---present new algorithms and give upper bounds, while this work focuses on \ood generalization and proves lower bounds which match rates already known to be achievable for more general classes of losses \citep{hazan2007logarithmic, abernethy2008optimal, suggala2020online}
, implying that existing algorithms (ERM and a noisy variant) are already optimal.

\section{Conclusion and Future Directions}

This work presents the first formal results demonstrating an exponential computational gap between interpolation and extrapolation in domain generalization, a claim which has until now only been given vague intuitive justification.
Perhaps more importantly, we've shown that ERM remains statistically minimax-optimal for both tasks---given the observed failure of ERM in practice, this suggests that there is quite a bit more subtlety to distribution shift in the real world. Taken together, our results present strong evidence that the ``likelihood reweighting'' model of distribution shift, while perhaps appropriate for specific settings involving sub-populations, might not be appropriate for the more general study of extrapolation to new domains. It could instead be beneficial to reconsider existing notions of inter- and extrapolation---particularly those involving linearity or generic likelihood reweighting---in the context of online learning, where the notions of regret and stochastic adversaries allow for more a nuanced study of statistical and algorithmic complexity.

We see two important directions for further research. First, the proposed \good generalization game serves as a standalone framework for the theoretical analysis of learning algorithms. As discussed in Section~\ref{sec:fixed-baseline}, considering regret in the online setting provides a more nuanced signal of an algorithm's expected performance, especially when we are not too worried about the \emph{literal worst case test distribution}. We hope that this new perspective will better enable future work to provide formal \ood generalization guarantees for their proposed methods. We note that this work considers only strongly convex functions, but using the same techniques one could extend the analysis to more general classes such as all convex losses; this setting might eliminate the statistical complexity gap and could lead to additional insight into the differences between inter- and extrapolation.

Second, there still remains significant flexibility in how we define ``interpolation'' and ``extrapolation'' with respect to training environments; we consider one specific notion in this work, and we show that ERM remains optimal---implying that alternative formulations may be preferable. However, it seems likely that different restrictions on the adversary could allow for stronger generalization guarantees. Furthermore, our analysis reveals that the \emph{geometry of the environmental loss functions} is a critical element for generalization. This suggests additional improvements can be achieved with careful representation learning.

\section*{Acknowledgements}
We thank Zack Lipton for his feedback on the framing of this work.

\newpage
\bibliography{main}
\bibliographystyle{icml2021}

\newpage
\onecolumn
\appendix

\section{Proof of Theorem 1}
\label{sec:thm1-proof}
\convexlower*

\begin{proof}
Define $F_t(z) = \sum_{s=1}^t f_s(z)$; since each $f$ is convex, this sum is convex as well. Let $\betastarprev$ be the minimizer of $F_{t-1}$ (by Lemma~\ref{lemma:sum-loss-t-times-convex}, this will lie in $B$), and let $z\in B$ be arbitrary. Finally, note that $\nabla^2 F_t \preceq t\sigmax I$. Then we have the following Taylor expansion:
\begin{align*}
    F_t(z) &= F_{t-1}(z) + f_t(z) \\
    &= F_{t-1}(\betastarprev + (z - \betastarprev)) + f_t(z) \\
    &\leq F_{t-1}(\betastarprev) + \nabla F_{t-1}(\betastarprev)^T (z - \betastarprev) + \frac{(t-1)\sigmax}{2} \twonormsq{z - \betastarprev} + f_t(z) \\
    &= F_{t-1}(\betastarprev) + \frac{(t-1)\sigmax}{2} \twonormsq{z - \betastarprev} + f_t(z),
\end{align*}
where we have used the fact that $\nabla F_{t-1}(\betastarprev) = 0$ by definition. Thus,
\begin{align}
\label{eq:regret-lower-bound-1}
    \sum_{s=1}^t f_s(\hatbeta_s) - F_t(z) &\geq \left( \sum_{s=1}^{t-1} f_s(\hatbeta_s)  - F_{t-1}(\betastarprev) \right) + (f_t(\hatbetat) - f_t(z) - \frac{(t-1)\sigmax}{2} \twonormsq{z - \betastarprev}).
\end{align}
Then we can write
\begin{align*}
    V_t &= \min_{\hatbeta_1\in B} \max_{\lambda_1} \ldots \min_{\hatbetat\in B} \max_{\lambda_t, z\in B} \left( \sum_{s=1}^t f_t(\hatbetat) - F_t(z) \right) \\
    &\geq \min_{\hatbeta_1\in B} \max_{\lambda_1} \ldots \min_{\hatbeta_{t-1}\in B} \max_{\lambda_{t-1}} \biggl[ \left( \sum_{s=1}^{t-1} f_s(\hatbeta_s)  - F_{t-1}(\betastarprev) \right)\\
    &\quad+ \min_{\hatbetat\in B} \max_{\lambda_t, z\in B} \left( f_t(\hatbetat) - f_t(z) - \frac{(t-1)\sigmax}{2} \twonormsq{z - \betastarprev} \right) \biggr].
\end{align*}

Thus, by lower bounding the second term, we can unroll the recursion and lower bound the total regret. In particular, showing a bound of $\Omega(\frac{1}{t})$ will result in an overall regret lower bound of $\Omega(\log T)$, which would imply that ERM achieves minimax-optimal rates for OOD generalization (this is also how we prove Corollary~\ref{cor:erm-opt}).

We proceed by lower bounding the inner optimization term. We consider two possibilities for the choice of $\hatbetat$. Suppose $\twonormsq{\hatbetat - \betastarprev} \geq \frac{g^2}{8t\sigmax^2}$. Then by choosing $z = \betastarprev$ the inner term can be lower bounded by $\min_{\hatbetat\in B} \max_{\lambda_t} \left( f_t(\hatbetat) - f_t(\betastarprev) \right)$. Taylor expanding $f_t$ around $\betastarprev$ gives
\begin{align*}
    f_t(\hatbetat) - f_t(\betastarprev) &\geq \nabla f_t(\betastarprev)^T(\hatbetat - \betastarprev) + \frac{\sigmin}{2} \twonormsq{\hatbetat - \betastarprev}.
\end{align*}
By Lemma~\ref{lemma-prev-minimizer-zero-gradient}, the adversary can always play $\lambda_t$ such that $\nabla f_t(\betastarprev) = 0$. So plugging this in we get
\begin{align*}
    \min_{\hatbetat\in B} \max_{\lambda_t} \left( f_t(\hatbetat) - f_t(\betastarprev) \right) &\geq \frac{\sigmin}{2} \twonormsq{\hatbetat - \betastarprev} \\
    &\geq \frac{g^2\sigmin}{16t\sigmax^2}.
\end{align*}
Now consider the case where $\twonormsq{\hatbetat - \betastarprev} < \frac{g^2}{8t\sigmax^2}$. Suppose the adversary plays any $\lambda_t$ such that $\twonorm{\nabla f_t(\hatbetat)} \geq g$ (by definition, such a choice is always possible). Here we again split on cases, considering the possible values of $\nabla f_t(\betastarprev)^T(\hatbetat - \betastarprev)$:

\textbf{Case 1: $\nabla f_t(\betastarprev)^T(\hatbetat - \betastarprev) \geq \frac{g^2\sigmin}{16t\sigmax^2}$}

\noindent Following the same steps as previously, we find the lower bound 
\begin{align*}
    f_t(\hatbetat) - f_t(\betastarprev) &\geq \nabla f_t(\betastarprev)^T(\hatbetat - \betastarprev) + \frac{\sigmin}{2} \twonormsq{\hatbetat - \betastarprev} \\
    &\geq \nabla f_t(\betastarprev)^T(\hatbetat - \betastarprev) \\
    &\geq \frac{g^2\sigmin}{16t\sigmax^2}.
\end{align*}

\textbf{Case 2: $\nabla f_t(\betastarprev)^T(\hatbetat - \betastarprev) < \frac{g^2\sigmin}{16t\sigmax^2}$}

In this case the lower bound follows directly from Lemma~\ref{lemma-thm-1-case-2}.

Thus the lower bound is shown in all cases; it follows that
\begin{align*}
    V_t &\geq \min_{\hatbeta_1\in B} \max_{\lambda_1} \ldots \min_{\hatbeta_{t-1}\in B} \max_{\lambda_{t-1}} \left[ \left( \sum_{s=1}^{t-1} f_s(\hatbeta_s)  - F_{t-1}(\betastarprev) \right)
    + \frac{g^2\sigmin}{16t\sigmax^2} \right] \\
    &= \min_{\hatbeta_1\in B} \max_{\lambda_1} \ldots \min_{\hatbeta_{t-1}\in B} \max_{\lambda_{t-1}} \left[  \sum_{s=1}^{t-1} f_s(\hatbeta_s)  - F_{t-1}(\betastarprev) \right] + \frac{g^2\sigmin}{16t\sigmax^2} \\
    &= V_{t-1} + \frac{g^2\sigmin}{16t\sigmax^2}.
\end{align*}
Expanding the recursion finishes the proof.
\end{proof}

\section{Proof of Existence for Theorem~\ref{thm:no-sublinear}}
\label{sec:thm234-proof}
We restate Theorem~\ref{thm:no-sublinear} for convenience:
\deterministic*
In the main body, we prove the primary claim. Here we include proof of the existence of a regression task over a set of distributions which induces the loss functions we construct in our proof.
\begin{proof}
Suppose we are regressing labels $y\in\R$ on observations $z\in\R^2$ with squared loss. We’ll define our classifier with a parameter $\beta$ such that given an observation $(z_1,z_2)$ we predict $\beta^2 z_1 + \beta z_2$. This is of course an unusual regression setup, but we’re just giving an existence proof for a simple lower bound.

The first environment will assign all its probability mass to a single example $(z_1, z_2, y) = (0, 1, 0)$. Thus, if we choose a parameter $\beta$, in this environment we will suffer risk $\E[(\beta z_1^2 + \beta z_2 - y)^2] = \beta^2$. This produces the first environment, loss $f_{e_1}(\beta) = \beta^2$.

We define the second environment as having two possible samples: one is $(z_1, z_2, y) = (0, \sqrt{\frac{2\alpha+1}{2\alpha}}, 0)$ and the other is $(z_1, z_2, y) = (\sqrt{\frac{2\alpha+1}{2\alpha}}, 0, 0)$. Thus, the first sample induces loss $\frac{2\alpha+1}{2\alpha}\beta^2$, and the second induces loss $\frac{2\alpha+1}{2\alpha}\beta^4$. Now for the probabilities: we assign probability $\frac{1}{2\alpha+1}$ to the first point and $\frac{2\alpha}{2\alpha+1}$ to the second point. Clearly these sum to 1, and taking the expectation over losses we see that the overall risk is $\beta^4 + \frac{1}{2\alpha}\beta^2$, as desired.
\end{proof}

\section{Lemmas}
\fubinis*
\begin{proof}
Using Fubini's theorem, we have
\begin{align*}
    \calR^\lambda(\beta) &= \int_{\calX\times\calY} \left[ \sum_{e\in\calE} \lambda_e p^e(x,y) \right] \ell(\beta, (x, y))\ d(x,y)\\
    &= \sum_{e\in\calE} \lambda_e \int_{\calX\times\calY} p^e(x,y) \ell(\beta, (x, y))\ d(x,y) \\
    &= \sum_{e\in\calE} \lambda_e \calR^e(\beta). \qedhere
\end{align*}
\end{proof}
\begin{lemma}
\label{lemma:sum-loss-t-times-convex}
For any $F_t = \sum_{s=1}^t f_t$, there exist convex coefficients $\hat\lambda$ such that 
\begin{align*}
    F_t &= t\sum_{e\in\calE} \hat\lambda_e f_e.
\end{align*}
\end{lemma}
\begin{proof}
Every loss function $f_t$ can be written as a convex combination of the original environment losses:
\begin{align*}
    f_t &= \sum_{e\in\calE} \lambda_{t,e} f_e.
\end{align*}
So, write
\begin{align*}
    F_t &= \sum_{s=1}^t f_t
    = \sum_{s=1}^t \sum_{e\in\calE} \lambda_{t,e} f_e
    = \sum_{e\in\calE} \left(\sum_{s=1}^t \lambda_{t, e}\right) f_e.
\end{align*}
Clearly, $\sum_{e\in\calE} \left(\sum_{s=1}^t \lambda_{t, e}\right) = t$. So, defining $\hat\lambda_e := \frac{1}{t} \left(\sum_{s=1}^t \lambda_{t, e}\right)$ gives the desired result.
\end{proof}

\begin{lemma}
\label{lemma-prev-minimizer-zero-gradient}
For any solution $\betastarprev$ which minimizes the sum of previously seen losses $F_{t-1}$, there exists a convex combination of losses $f_t$ playable by the adversary for which $\nabla f_t(\betastarprev) = 0$.
\end{lemma}
\begin{proof}
By Lemma \ref{lemma:sum-loss-t-times-convex}, we can write $F_{t-1} = (t-1)\sum_{e\in\calE} \hat\lambda_e f_e$ for some convex coefficients $\hat\lambda$. Define $f_t = \sum_{e\in\calE} \hat\lambda_e f_e = \frac{1}{t-1} F_{t-1}$. Since $\betastarprev$ minimizes $F_{t-1}$ it follows that
\begin{align*}
    \nabla f_t(\betastarprev) =  \frac{1}{t-1} \nabla F_{t-1}(\betastarprev) = 0.
\end{align*}
\end{proof}

\begin{lemma}
\label{lemma-thm-1-case-2}
Let $\hatbetat$, $\lambda_t$ be such that $\twonormsq{\hatbetat - \betastarprev} < \frac{g^2}{8t\sigmax^2}$ and $\twonorm{\nabla f_t(\hatbetat)} \geq g$. Define $z := \betastarprev - c\nabla f_t(\hatbetat)$, where $c := 1/2t\sigmax$. If $\nabla f_t(\betastarprev)^T(\hatbetat - \betastarprev) < \frac{g^2\sigmin}{16t\sigmax^2}$, then
\begin{align*}
    f_t(\hatbetat) - f_t(z) - \frac{(t-1)\sigmax}{2}\twonormsq{z - \betastarprev} &\geq \frac{g^2\sigmin}{16t\sigmax^2}.
\end{align*}
\end{lemma}
\begin{proof}
Expanding $f_t$ around $\hatbetat$,
\begin{align*}
    f_t(\hatbetat) - f_t(z) \geq -\nabla f_t(\hatbetat)^T (z - \hatbetat) - \frac{\sigmax}{2}\twonormsq{z - \hatbetat},
\end{align*}
which gives
\begin{align}
    \label{eq:begin-lower-bound}
    &f_t(\hatbetat) - f_t(z) - \frac{(t-1)\sigmax}{2}\twonormsq{z - \betastarprev} \nonumber \\
    &\geq \nabla f_t(\hatbetat)^T (\hatbetat - z) - \frac{\sigmax}{2} \left(\twonormsq{z - \hatbetat} + (t-1)\twonormsq{z - \betastarprev} \right) \nonumber \\
    &= \nabla f_t(\hatbetat)^T(\hatbetat - \betastarprev + c\nabla f_t(\hatbetat)) - \frac{\sigmax}{2} \left( \twonormsq{\betastarprev - \hatbetat - c\nabla f_t(\hatbetat)} + (t-1) \twonormsq{c\nabla f_t(\hatbetat)} \right).
\end{align}
By the triangle inequality,
\begin{align*}
    \twonorm{\betastarprev - \hatbetat - c\nabla f_t(\hatbetat)} &\leq \twonorm{\betastarprev - \hatbetat} + c\twonorm{\nabla f_t(\hatbetat)},
\end{align*}
and therefore
\begin{align*}
    \frac{1}{2}\twonormsq{\betastarprev - \hatbetat - c\nabla f_t(\hatbetat)} &\leq \twonormsq{\betastarprev - \hatbetat} + c^2\twonormsq{\nabla f_t(\hatbetat)}.
\end{align*}
Continuing with the lower bound in Equation~\ref{eq:begin-lower-bound},
\begin{align*}
    &\geq \nabla f_t(\hatbetat)^T(\hatbetat - \betastarprev) + c\twonormsq{\nabla f_t(\hatbetat)} - \sigmax \left( \twonormsq{\betastarprev - \hatbetat} + c^2\twonormsq{\nabla f_t(\hatbetat)} \right) - \frac{(t-1)\sigmax c^2}{2} \twonormsq{\nabla f_t(\hatbetat)} \\
    &\geq \nabla f_t(\hatbetat)^T(\hatbetat - \betastarprev) + \left( c - \frac{1}{8t\sigmax} -  \frac{(t+1)c^2\sigmax}{2} \right) \twonormsq{\nabla f_t(\hatbetat)},
\end{align*}
where we've used the upper bound on $\twonormsq{\betastarprev - \hatbetat}$ and simplified. Recalling that $c = \frac{1}{2t\sigmax}$ and noting that $\frac{t+1}{t^2} \leq \frac{2}{t}$,
\begin{align*}
    &=  \nabla f_t(\hatbetat)^T(\hatbetat - \betastarprev) + \left(\frac{1}{2t\sigmax}-  \frac{1}{8t\sigmax} - \frac{(t+1)}{8t^2\sigmax}\right) \twonormsq{\nabla f_t(\hatbetat)} \\
    &\geq \nabla f_t(\hatbetat)^T(\hatbetat - \betastarprev) + \frac{\twonormsq{\nabla f_t(\hatbetat)}}{8t\sigmax} \\
    &\geq \nabla f_t(\hatbetat)^T(\hatbetat - \betastarprev) + \frac{g^2}{8t\sigmax}.
\end{align*}
By strong convexity,
\begin{align*}
    (\nabla f_t(\betastarprev) - \nabla f_t(\hatbetat))^T (\betastarprev - \hatbetat) \geq \sigmin \twonormsq{\betastarprev - \hatbetat},
\end{align*}
and therefore
\begin{align*}
    \nabla f_t(\hatbetat)^T(\hatbetat - \betastarprev) &\geq \sigmin \twonormsq{\betastarprev - \hatbetat} - \nabla f_t(\betastarprev)^T(\hatbetat - \betastarprev) \\
    &> - \frac{g^2\sigmin}{16t\sigmax^2},
\end{align*}
where the second inequality is due to the assumption in the Lemma statement. Plugging this in above gives
\begin{align*}
    \nabla f_t(\hatbetat)^T(\hatbetat - \betastarprev) + \frac{g^2}{8t\sigmax} &> - \frac{g^2\sigmin}{16t\sigmax^2} + \frac{g^2}{8t\sigmax} \\
    &\geq \frac{g^2\sigmin}{8t\sigmax^2} - \frac{g^2\sigmin}{16t\sigmax^2} \\
    &= \frac{g^2\sigmin}{16t\sigmax^2},
\end{align*}
completing the proof.
\end{proof}

\end{document}